\documentclass{article}
\usepackage{aistats2016}
\usepackage{times}
\usepackage{makeidx}  
\usepackage{multirow}      
\usepackage{graphicx}
\usepackage[T1]{fontenc}
\usepackage{bbm}
\usepackage{algorithm, algpseudocode}
\usepackage{amsmath,calc}
\usepackage{booktabs}
\usepackage{natbib}
\newtheorem{thm}{Theorem}

\newtheorem{definition}{Definition}
\newtheorem{proof}{Proof}

\newcommand{\yn}{Y_n}

\newcommand{\iplus}{\mathcal{I}^+}
\newcommand{\iminus}{\mathcal{I}^-}
\newcommand{\PEm}{\mathcal{P}^E_\text{max}}

\begin{document}

%

%

\twocolumn[

\aistatstitle{Learning Optimized Or's of And's}

\aistatsauthor{ Tong Wang \And Cynthia Rudin }

\aistatsaddress{ CSAIL and Sloan, MIT \\ tongwang@mit.edu\And CSAIL and Sloan, MIT \\ rudin@mit.edu}
 ]

\begin{abstract}
\textit{Or}'s of \textit{And}'s (OA) models are comprised of a small number of disjunctions of conjunctions, also called disjunctive normal form. An example of an OA model is as follows: If ($x_1 = $ `blue' {\em AND} $x_2=$ `middle') {\em OR} ($x_1 = $ `yellow'), then predict $Y=1$, else predict $Y=0$. Or's of And's models have the advantage of being interpretable to human experts, since they are a set of conditions that concisely capture the characteristics of a specific subset of data. 
We present two optimization-based machine learning frameworks for constructing OA models, Optimized OA (OOA) and its faster version, Optimized OA with Approximations (OOAx). We prove theoretical bounds on the properties of patterns in an OA model. We build OA models as a diagnostic screening tool for obstructive sleep apnea, that achieves high accuracy with a substantial gain in interpretability over other methods.  
\end{abstract}

\section{Introduction}
We present mathematical programming formulations for producing \textit{Or}'s of \textit{And}'s, which are sparse disjunctive normal form (DNF) expressions. An OA model might say, for instance, consumers who are female {\em AND} single,  {\em AND} younger than 35 years old, {\em OR} married {\em AND} earn more than \$100K per year, are likely to purchase a product. In creating predictive models for healthcare, marketing, sociology, and in other domains, two aspects have long since been of interest: logical forms, and sparsity \citep[see, e.g.][]{dawes1979robust,johnson2015logic,miller1956magical}. For example, physicians use sparse, easily checkable sets of conditions (symptoms, observations) to classify patients as to whether they have a disease. DNF formulae are particularly useful as \textit{screening} models, where patients who do not meet the Or of And's criteria are not considered for further testing. In marketing, DNF is also called ``disjunctions of conjunctions" or ``non-compensatory decision rules." Marketing researchers strongly hypothesize that consumers use simple rules to screen products, and they would consider purchasing only the products in this \textit{consideration set} to reduce the cognitive load of considering all products. The consideration set may be precisely an Or's of And's classifier \citep{hauser2010disjunctions, gilbride2004choice}. 

Despite the efforts that theoretical communities have placed on learning DNF \citep[e.g.,][]{klivans01, littlestone1988learning, ehrenfeucht1989general}, the algorithms designed for inductive logic programming that essentially produce DNF \citep{muggleton1994inductive,lavrac1994inductive}, the associative classification algorithms \citep{ma1998integrating,han2000mining,li2001cmar,yin2003cpar, chen2006new,cheng2007discriminative}, and rule induction methods \citep[e.g.,][]{cohen1995fast},
there has been little in the way of algorithms designed for applications where cognitive simplicity is first and foremost \citep[with some exceptions, like][which we discuss later]{hauser2010disjunctions}. 
For instance, \cite{yin2003cpar} reported that the average number of rules used by CMAR was 305 and CPAR had 244 rules on average, on 26 datasets from UCI ML Repository \citep{Lichman:2013}; these are not cognitively simple models. Besides, all of the algorithms discussed above use greedy approximations, which hurts accuracy and sparsity. For instance, RIPPER employs local greedy splitting, meaning that a mistake at the beginning is difficult to undo. Inductive logic programming starts with a collection of rules and locally combines them. The associative classification methods also follow separate-and-conquer or covering strategies. Unlike these methods, our methods aim to produce cognitive simple models, and do not use greedy approximations or similar heuristics. 

The closest work to ours are that of \citet{hauser2010disjunctions} in the marketing literature, and \citet{wang2015or} on Bayesian modeling of DNF formulae. \citet{hauser2010disjunctions} pre-mines the rules and then uses an integer program equivalent to a set-covering problem. They try to minimize the error of coverage while favoring patterns that fit the largest subset of data. 
The work of \citet{wang2015or} has the advantage of a Bayesian interpretation of the prior parameters, but a disadvantage in that the analytical bounds on the maximum a posteriori solution of the Bayesian method are weaker than those we present in this paper for the minimizer of the optimization problem. 

Our goal is cognitive simplicity, as well as predictive accuracy. We choose mathematical programming (mixed-integer and integer linear programming -- MIP and ILP) and rule mining to form our models.  Using these tools have several benefits, namely flexibility on the user's side on the objective and constraints, fast solvers that have been improving exponentially over recent years, and a guarantee on the optimality of the solution. We improve computation also using statistical approximations. In one of our algorithms, OOAx, we first mine rules and design an ILP to choose the subset of rules to form the OA model. This is a statistical assumption that dramatically speeds up computation, but we can show (in Theorem \ref{thm:pos_supp}) that as long as we mine all rules with sufficiently high support, the optimal solution will be attained anyway. We also present various bounding conditions on the optimal solution.

\section{Optimized Or's of And's}\label{sec:model1}
Let us discuss the first framework for learning Or's of And's, called Optimized Or's of And's (OOA).
We work with a data set $S = \{(X_n,Y_n)\}_n^N$, consisting of $N$ examples with $J$ attributes of mixed type. $\yn\in\{1,-1\}$ represents the labels. 
Numerical attributes are indexed by index set $\mathcal{J}_\text{n}$ and categorical attributes are indexed by $\mathcal{J}_\text{c}$. The $j$-th attribute of the $n$-th example is denoted as $X_{nj}$.

An OA classifier consists of a set of patterns that characterize a single class, here, the positive class. Each pattern is a conjunction of conditions (literals), and the number of conditions is called the \textit{length} of a pattern. For example. the length of pattern ``age $ \geq 30$ {\em AND} has hypertension {\em AND} is female'' is 3. Let $z$ denote a pattern, and  $\mathbbm{1}_z(X)$ indicate if $X$ satisfies pattern $z$. $A$ represents a set of patterns. An OA classifier built on $A$ is denoted as $f_A$:
\begin{equation}
f_A(X)=  \begin{cases} 1 & \exists z \in A, \text{s.t.}\mathbbm{1}_z(X)=1 \\  0 & \text{otherwise.} \end{cases}
\end{equation}

\subsection{MIP Formulation}
We formulate a mixed integer program to generate a pattern set containing numerical and categorical attributes. The MIP uses the following objective $L(A)$ to minimize the training error while maintaining sparseness of the model.
\begin{eqnarray}
\small
\resizebox{.9\hsize}{!}{$L(A)=\frac{\#\text{error}(A)}{N} + C_1\#\text{literals}(A)  + C_2 \#\text{patterns}(A)$}.\label{eqn:oa_obj}
\end{eqnarray}
The first term in the objective is the loss function, which counts the number of classification errors. The regularization terms include 1) the total number of literals in $A$, denoted as $\#\text{literals}(A)$, which is the sum of the length of each pattern in $A$, and 2) the total number of patterns in $A$, denoted as $\#\text{patterns}(A)$. The two terms are scaled by parameters $C_1$ and $C_2$ to penalize  the complexity of the model. $C_1$ represents the percentage of training errors the user is willing to trade in order to reduce a pattern by one literal. Similarly, $C_2$ represents the percentage of training errors a user needs to trade to reduce one pattern. 
 A user can tune $C_1$ and $C_2$ to influence the shape of the output.

Now we explain how the constraints work. A challenging part is to deal with both numerical and categorical attributes in the MIP. For numerical attributes, the MIP needs to select the upper and lower boundary to form a range; for categorical attributes, it needs to select a category for a literal. Simultaneously, the MIP needs to decide for each example if it satisfies the literals and patterns. All of the constraints are linear in the decision variable, to ensure a duality gap or proof of optimality on the solution we obtain. 
\subsubsection{Literals for Numerical Attributes}
For a numerical attribute, a literal $j$ (for simplicity, when we refer to literal $j$, we mean a literal containing attribute $j$) has the form ``$l_{kj} \leq X_{\cdot j} \leq u_{kj}$'', where $u_{kj}$ and $l_{kj}$ represent the upper and lower boundary of the range in literal $j$. $k$ is a pattern index. For each example $X_n$, let $\hat{u}_{nkj}\in\{0,1\}$ indicate if $X_{nj}$ satisfies the upper bound $u_{kj}$ and $\hat{l}_{nkj}\in \{0,1\}$ indicate if $X_{nj}$ satisfies the lower bound $l_{kj}$. That is,
$\hat{u}_{nkj} = 1$ if $X_{nj}\leq u_{kj}$, and $\hat{l}_{nkj} = 1$ if $X_{nj}\geq l_{kj}$. Using a big $M$ formulation, we obtain the following constraints, that for $\forall n,k, \forall j \in \mathcal{J}_\text{n}$,
\begin{align}
& u_{kj} - X_{nj} \leq M \hat{u}_{nkj}, \label{eqn:uhat1}\\
& u_{kj} - X_{nj} \geq M (\hat{u}_{nkj} -1) + \epsilon,\label{eqn:uhat2}\\
& X_{nj} - l_{kj} \leq M \hat{l}_{nkj}, \label{eqn:lhat1}\\
& X_{nj} - l_{kj} \geq  M(\hat{l}_{nkj}-1) + \epsilon, \label{eqn:lhat2} \\
& u_{kj} \leq U_j \label{eqn:u_ub}\\
& l_{kj} \geq L_j \label{eqn:l_lb}
\end{align}
A small number $\epsilon$ is used to force $\hat{u}_{nkj}=1$ when $X_{nj} = u_{kj}$, and force $\hat{l}_{nkj}=1$ when $X_{nj} = l_{kj}$. $U_j$ and $L_j$ denote the maximum and minimum value of attribute $j$. Constraints (\ref{eqn:u_ub}) and (\ref{eqn:l_lb}) bound on $u_{kj}$ and $l_{kj}$ to ensure a bounded $M$ for computation efficiency. 

It is possible that the upper and lower bounds apply to all examples, when both constraints (\ref{eqn:u_ub}) and (\ref{eqn:l_lb}) are binding. In that case, the literal does not have any classification power. We need numerical literals that are meaningful, or what we call \textit{substantive}, that their ranges  only apply to a subset of training examples. This means at most one of constraints (\ref{eqn:u_ub}) and (\ref{eqn:l_lb}) can be binding.
Let a binary variable $\delta_{kj}$ indicates if literal $j$ in pattern $k$ is substantive. We use a big $M$ formulation to construct the constraints. 
Therefore $\forall k, \forall j \in \mathcal{J}_\text{n}$,
\begin{align}
& M \delta_{kj}  \geq  U_j - u_{kj},  \label{eqn:u_dtkj}\\
& M\delta_{kj} \geq l_{kj}-L_j  , \label{eqn:l_dtkj} \\
& M(1 - \delta_{kj}) \geq (u_{kj} - l_{kj}) - (U_j - L_j) + \epsilon, \label{eqn:ul_dtkj}
\end{align}
Constraint (\ref{eqn:u_dtkj}) means $\delta_{kj} = 1$ if $u_{kj}<U_j$. Constraint (\ref{eqn:l_dtkj}) means $\delta_{kj} = 1$ if $l_{kj}>L_j$.  
(\ref{eqn:ul_dtkj}) forces $\delta_{kj} = 0$ when $u_{kj} = U_j$ and $l_{kj} = L_j$, i.e., literal $j$ is non-substantive.

\subsubsection{Literals for Categorical Attributes}
For a categorical attribute, a literal $j$ has the form ``$X_{\cdot j} = \text{the $v$-th category}$'', where $v \in \{1,...V_j\}$ is an index for categories of attribute $j$ and $V_j$ is the total number of categories. We use $o_{kjv} \in \{0,1\}$ to indicate whether the $v$-th category of attribute $j$ is present in literal $j$ of pattern $k$. 
To determine if $X_n$ satisfies the condition in literal $j$, let $\hat{o}_{nkj} \in \{0,1\}$ indicate if $X_{nj}$ equals the category contained in this literal. We binary code $X_{nj}$ into $X_{njv}$ such that $X_{njv} = 1$ if $X_{nj}$ takes the $v$-th category of attribute $j$.  Therefore $
\hat{o}_{nkj} = 1$ if and only if there exists $v \in \{1,...V_j\}$ such that $o_{kjv} = 1$ and $X_{njv}=1$.
We formulate it as the following. For $\forall n,k, \forall j \in \mathcal{J}_\text{c}$,
\begin{align}
&\hat{o}_{nkj}  \leq \sum_{v}X_{njv}o_{kjv}, \label{eqn:cate_active1}\\
&V_j\hat{o}_{nkj}  \geq \sum_{v}X_{njv}o_{kjv}, \label{eqn:cate_active2}\\
& \sum_{v \in \{1,...,V_j\}}o_{kjv} \leq 1.\label{eqn:oneattribute}
\end{align}
(\ref {eqn:oneattribute}) ensures each categorical literal contains at most one value. This constraint forces a pattern to have a fixed form. If we remove the constraint and allow a literal to take multiple values, a pattern could have the following form: ``$5 \leq X_1 \leq 20$ AND $X_2 = $ red or blue.'' It depends on the application and users' preference as to whether leave this constraint in the MIP. The model will work the same without changing the rest of the formulation.  

We define that a categorical literal $j$ is \textit{substantive}, if there exists some $v\in\{1,...,V_j\}$ such that $o_{kjv} = 1$, indicated by $\delta_{kj}\in\{0,1\}$. For $\forall k, \forall j \in \mathcal{J}_\text{c}$,
\begin{align}
& \delta_{kj} \leq \sum_{v \in \mathcal{V}_j}o_{kjv},  \label{eqn:cate_attribute1}\\
& V_j\delta_{kj} \geq \sum_{v \in \mathcal{V}_j}o_{kjv}.   \label{eqn:cate_attribute2} 
\end{align}
\subsubsection{Counting Classification Errors}
Given the literals, $X_n$ satisfies pattern $k$ if and only if it satisfies every literal in the pattern.
For categorical attributes, we consider both cases where a literal is substantive, i.e., $\delta_{kj} = 1$, and we need $\hat{o}_{nkj} = 1$; or non-substantive, i.e. $\delta_{kj} = 0$, and $\hat{o}_{nkj}$ is always 0 for all $v$. For numerical attributes, when the literal is substantive, the MIP needs to check if a data point satisfies both upper and lower bounds of the range, indicated by $\hat{u}_{nkj}$ and $\hat{l}_{nkj}$; when the literal is non-substantive, i.e., $u_{kj} = U_j$ and $l_{kj} = L_j$, then $\hat{u}_{nkj} = 1$ and $\hat{l}_{nkj} = 1$ for all $X_n$. Using $\omega_{nk}\in\{0,1\}$ to indicate if $X_n$ satisfies pattern $k$, the above conditions can be formulated below. $\forall n,k$, 
\begin{align}
&\zeta_k + \sum_{j\in \mathcal{J}_\text{n}} \left( \hat{u}_{nkj} + \hat{l}_{nkj}\right) + \sum_{j\in \mathcal{J}_\text{c}}\left(\hat{o}_{nkj} + 1 - \delta_{kj}\right) - \notag \\
& \quad \quad \quad \quad \quad \quad \quad \quad\quad \quad  \left(2|\mathcal{J}_\text{n}| + |\mathcal{J}_\text{c}|\right)  \leq \omega_{nk}, \label{eqn:pattern1}\\
& (2|\mathcal{J}_\text{n}|+|\mathcal{J}_\text{c}| + 1)\omega_{nk} \leq  \zeta_k + \sum_{j\in \mathcal{J}_\text{n}} \left( \hat{u}_{nkj} + \hat{l}_{nkj}\right) + \notag \\
& \quad \quad \quad \quad \quad \quad\quad \quad \quad \quad \quad\sum_{j\in \mathcal{J}_\text{c}}\left(\hat{o}_{nkj} + 1 - \delta_{kj} \right).\label{eqn:pattern2} 
\end{align}
Let $\xi_n\in\{0,1\}$ indicate if a classification error is made, which means either a positive data point does not satisfy any pattern, or a negative data point satisfies at least one pattern. In both cases $\xi_n = 1$. These two situations are captured by constraints (\ref{eqn:pos_error1}) and (\ref{eqn:neg_error1}).
\begin{align}
& \xi_n + \sum_k^K \omega_{nk} \geq 1, \forall n \in \iplus ,\label{eqn:pos_error1}\\
& K\xi_n \geq \sum_k^K \omega_{nk}, \forall n \in \iminus ,\label{eqn:neg_error1}
\end{align}
where $\iplus$ denotes the set of indices for positive examples and $\iminus$ denotes the set of indices for negative examples. $K$ is the upper bound on the number of patterns that we allow the solution to have. MIP creates this $K$ ``boxes'' that will be filled up as it searches in the solution space.

When the MIP is formulated, we do not know how many out of the $K$ ``boxes'' the MIP will use. Therefore we introduce binary variables $\zeta_k \in \{0,1\}$ to indicate if pattern $k$ is non-empty in the final pattern set, which means it contains at least one substantive literal. For $\forall k$,
\begin{align}
J\zeta_k \geq \sum_{j} \delta_{kj}. \label{eqn:zeta}
\end{align}
Since we are minimizing the total number of patterns in the objective, the constraint will always be binding.

\subsubsection{The Objective}
Now we reprsent the objective using decision variables introduced before. 
The MIP minimizes
\begin{equation*}
\frac{1}{N}\sum_{n=1}^N\xi_n + C_1\sum_{k}^{K}\sum_j^J\delta_{kj} + C_2 \sum_{k=1}^{K}\zeta_k
\end{equation*}
over variables $\omega_{nk}$,$u_{kj}$, $l_{kj}$, $\hat{u}_{nkj}$, $\hat{l}_{nkj}$, $o_{kjv}$, $\hat{o}_{nkj}$, $\delta_{kj}$, $\delta_j$, $\xi_n$, and $\zeta_k$,
such that they satisfy constraints (\ref{eqn:uhat1}) to (\ref{eqn:zeta}).  
The complexity of the MIP comes from three aspects, 1) choosing the upper and lower boundaries for ranges in numerical literals, and picking categories for categorical literals, 2) deciding for each example if it satisfies every literal and every pattern, and 3) deciding how many patterns are constructed from the $K$ ``boxes.''  There are in total $\mathcal{O}(NKJ)$ constraints and $\mathcal{O}(NKJ)$ decision variables for this MIP, though the full matrix of variables corresponding to the mixed integer programming formulation is sparse since most literals operate only on a small subset of the data. This formulation can be solved efficiently for small to medium sized datasets.  As the size of the dataset grows,  the computation gets  complicated. We might need a faster method that operates in an approximate way on a much larger scale, presented below. 

\section{Optimized Or's of And's with Approximations}
To speed up the learning process, we propose {\bf Optimized Or's of And's with Approximations} (OOAx), that separates from the optimization process, the first two previously mentioned aspects of complexity. OOAx uses a {\em pre-mining then selecting} approach. It takes advantage of mature pattern mining techniques to generate a set of patterns. Then a secondary criteria is applied to further screen the rules to form a candidate pattern set. 
Finally, an integer linear program (ILP) searches within these patterns set for an optimal set. This method consists of following three steps, {\em pattern mining}, {\em pattern screening} and {\em pattern selecting}.

\subsection{Pattern Mining}
There are many frequently used pattern mining methods such as FP-growth \citep{han2000mining}, Apriori \citep{agrawal1994fast}, Eclat \citep{zaki1997new}, etc. In our implementation, we use FP-growth in python \citep{borgelt2005implementation} that takes the binary coded data, and user specified minimum support and maximum length, to generate patterns that satisfy the two requirements. The algorithm runs sufficiently fast (usually less than a second for thousands of observations). 
%
Since the FP-growth algorithm handles binarized data, we discretize the numerical attributes by manually selecting thresholds for bins. For instance, $X = 3.5$ can be transformed into $ 2\leq X\leq 4$, etc. Note that there are other pattern mining techniques that handle real-valued variables. 
\subsection{Pattern Screening}
In the pattern mining step, the number of generated patterns is usually overwhelming for even a medium size data set. For instance, for the sleep apnea data set (which we will discuss in detail in the experiment sections) of size 1192 patients and 112 binary coded attributes, if the maximum length is 3 and the minimum support is 5\%, millions of patterns are generated. Ideally, we would like the candidate pattern set to contain thousands of patterns for computational convenience. Therefore, we use a secondary criteria to further screen the patterns.
\begin{equation}
\text{Score}(z) = \text{InfoGain}(S|z) - \gamma l_z.\label{eqn:criteria}
\end{equation}
This criteria considers the classification power of a pattern, measured by information gain $\text{InfoGain}(S|z)$, and the sparsity, measured by the length of the pattern  $l_z$.
Information gain of pattern $z$ on data $S$ is $\text{InfoGain}(S|z) = H(S) - H(S|z )$, where $H(S)$ is the entropy of $S$, written as $H(S) = -\sum_{i}P_i\log P_i$. $H(S|z)$ is the conditional entropy of $S$. Using this criteria, we select a set of candidate patterns $\mathcal{P}$ of size $K_{\mathcal{P}}$. 

To represent the sparseness of each pattern, we create a binary matrix $\bf{P}$ of size $K_{\mathcal{P}} \times J$, where each row represents which attribute is present in a pattern. For instance, $P_{kj}=1$  indicates that literal $j$ is substantive in pattern $k$, and $P_{kj}=0$ otherwise. We also need to determine for each example, which of the $K_{\mathcal{P}}$ patterns it satisfies. For a data set with $N$ examples, we create a matrix $\mathbf{W}$ of size $N \times K_{\mathcal{P}}$, where  the $k$-th element in the $n$-th row, $\omega_{nk}$, indicates if the $n$-th observation satisfies pattern $k$. Both matrices are pre-computed before the final step.

\subsection{Pattern Selecting}
The previous two steps greatly reduce the computational load by feeding the final step with a set of high quality candidate patterns. Now our goal is only to select an optimal set $A$ from the candidate set $\mathcal{P}$. We formulate an ILP using the same objective (\ref{eqn:oa_obj}), and present it below.
\begin{equation*}
\min_{\xi_n,\zeta_k} \frac{1}{N}\sum_{n=1}^N\xi_n + C_1\sum_{k}^{K_{\mathcal{P}}} \zeta_k l_k  + C_2 \sum_k^{K_{\mathcal{P}}} \zeta_k
\end{equation*}
such that
\begin{align}
& \xi_n + \sum_{k=1}^{K_{\mathcal{P}}}\omega_{nk}\zeta_k \geq 1, \forall n \in \iplus \label{eqn:pos_error2}\\
& K\xi_n \geq \sum_{k=1}^{K_{\mathcal{P}}}\omega_{nk}\zeta_k, \forall n \in \iminus \label{eqn:neg_error2}\\
&\xi_n,\zeta_k \in \{0,1\}.
\end{align}
The length of pattern $k$, $l_k$, can be pre-computed by $l_k = \sum_{j=1}^{J}P_{kj}$.  
Constraint (\ref{eqn:pos_error2}) means that an error occurs for a positive example if it does not satisfy any patterns that are selected. Constraint (\ref{eqn:neg_error2}) means that an error occurs for a negative example if it satisfies at least one pattern that is selected. This ILP only involves $\mathcal{O}(N)$ constraints and $\mathcal{O}(N)+\mathcal{O}(K_{\mathcal{P}})$ variables, which is much simpler than the MIP in an OOA framework.

The difference between OOA and OOAx method is that the latter avoids forming patterns in the optimization process, by handling it to other efficient off-the-shelf algorithms. Separating the mining step from the optimization problem renders more control to users over the quality and size of desired patterns. Users can manually modify the pre-mining and screening process by applying domain-specific minimum support, maximum length and secondary selection criteria. 
\vspace{-2mm}
\section{Analysis on Patterns and OA Models}\label{sec:analysis}
\vspace{-2mm}
In this section, we discuss the quality of patterns in an OA classifier. Certain properties of the patterns improve computation complexity. We also show the VC dimension of OA models and compare OA classifiers with other discrete classifiers (decision trees and random forests). Due to the page limit, some proofs are provided in the supplementary material. 
\subsection{Bounds on Patterns} 
\vspace{-2mm}
Define the {\em support set} of pattern $z$ over data set $S$ as
\begin{equation}
\mathcal{I}^S(z) = \{X|\mathbbm{1}_z(X) = 1,X \in S\},
\end{equation}
and the {\em support} of pattern $z$ over $S$ as
\begin{equation}
\text{supp}^S(z) = |\mathcal{I}^S(z)|.
\end{equation}
$\text{supp}^{S^+}(z)$ is called the {\em positive support} of $z$, which is the number of positive examples in $\mathcal{I}^S(z)$, and $\text{supp}^{S^-}(z)$ is called the {\em negative support} of $z$, which is the number of negative examples in $\mathcal{I}^S(z)$. 

An OA classifier is essentially an ensemble of weaker classifiers, patterns.
Including patterns with a low quality is expensive, and as we will prove, unnecessary. First we show in Theorem \ref{thm:neg_supp} that the optimal solution never includes a pattern with a high negative support. 
\begin{thm}\label{thm:neg_supp}
Take an OA model with regularization parameters $C_1$ and $C_2$. The OA model is trained on a data set $S$, consisting of $N$ examples, $N^+$ of which are positive examples. If $A^* \in \arg\min_{A} L(A)$, then for any $z \in A^*$, $\text{supp}^{S^-}(z) \leq  N^+ - N\left(C_1+C_2 \right)$.
\end{thm}
This means after pattern mining, we can safely reduce the pattern space by disregarding patterns with a negative support above $N^+ - N\left(C_1+C_2 \right)$. Similarly, we can also prove that if a pattern has a low positive support, removing it always achieves a better objective. Let $A_{\backslash z}$ denote the pattern set with pattern $z$ removed from $A$. 
\begin{thm}\label{thm:pos_supp}
Take an OA model with regularization parameters $C_1$ and $C_2$. The OA model is trained on a data set $S$, consisting $N$ examples. If $\text{supp}^{S^+}(z) \leq (C_1+C_2)N$, then $L(A_{\backslash z}) \leq L(A)$.
\end{thm}
It means we need not bother mining rules of low positive support.
Theorem \ref{thm:pos_supp} is a stronger statement than Theorem \ref{thm:neg_supp} since it provides a lower bound on positive support for patterns in all pattern sets and saying that removing a low supported pattern always improves the performance; while Theorem 2 only applies to optimal solutions.

With the above theoretical guarantees, we know it is safe to reduce the pattern space, by setting the minimum positive support to be $\left(C_1+C_2\right)N$ when we pre-mine the patterns and throwing away patterns with negative support higher than $ N^+ - N\left(C_1+C_2 \right)$ in the screening stage. This does not benefit an OOA framework as it directly forms patterns. But it provides strong computational motivation for pre-mining patterns in an OOAx framework.

The sparseness of a model is also associated with the number of patterns in an OA model. We prove in Theorem~\ref{thm:size} that the number of patterns in an optimal pattern set is upper bounded. 
\begin{thm}\label{thm:size}
Take an OA model with regularization parameters $C_1$ and $C_2$. The OA model is trained on a set of $N$ examples, $N^+$ of which are positive examples. If $A^* \in \arg\min_{A} L(A)$, then $|A^*| \leq \frac{N^+/N}{C_1+C_2}.$
\end{thm}
This theorem is meaningful not only for showing the simplicity of the output, but also gives us a suggestion for $K$ when we use the MIP in an OOA framework. Knowing that the optimal set can never be larger than $\frac{N^+/N}{C_1+C_2}$, we can safely set $K$ to be $\frac{N^+/N}{C_1+C_2}$. The smaller $K$ can be set, the better it is computationally for the MIP.
\subsection{VC Dimension of an OA classifier}
\vspace{-2mm}
Let us consider the VC dimension of hypothesis classes representing pattern sets selected from a pre-mined set $\mathcal{P}$. There are some results for k-DNF \citep{ehrenfeucht1989general} and monotone functions (that is, Boolean
functions that can be represented without negated literals) \citep{procaccia2006exact}. \citet{littlestone1988learning} has shown that the class of $k$-term monotone $l$-DNF formulas (i.e., with monomials containing at most $l$ variables) has VC dimension at least $lk\lfloor\log(\frac{n}{m})\rfloor$, where $l\leq m \leq n$, and $k \leq {m \choose l}$. However, his theorem does not have the constraint that the patterns come from a fixed pattern set. 

Let $\mathcal{S} = \mathbbm{R}^J$ represent the complete set of all possible data that could be constructed from $J$ attributes. To compute the VC dimension, we introduce the following definition. 
\begin{definition}\label{def:es}
An \textbf{Efficient Set} of $\mathcal{P}$ is a set of patterns where the support set of each pattern is not a subset of the rest of the efficient set, i.e.,
$$\mathcal{P}^E = \{z | z\in \mathcal{P}, \mathcal{I}^{\mathcal{S}}(z) \not\subset \mathcal{I}^{\mathcal{S}}({\mathcal{P}^E_{\backslash z})\}}.$$
\end{definition}
This means for any pattern $z$ in $\mathcal{P}^E$, there exists data points that satisfy only $z$ and none of the rest of the patterns in $\mathcal{P}^E$.
We call the efficient set with the maximum number of patterns the \textbf{Maximum Efficient Set} of $\mathcal{P}$, denoted as $\mathcal{P}^E_\text{max}$.
We claim that the VC dimension of OOAx learned from $\mathcal{P}$ depends on the size of $\mathcal{P}^E_\text{max}$, stated as the following. 
\begin{thm}\label{thm:vc}
The VC dimension of an OA classifier $f$ built from $\mathcal{P}$ equals the size of the maximum efficient set of $\mathcal{P}$: $$\text{VCdim}(f):= |\mathcal{P}^E_\text{max}|.$$
\end{thm}
\begin{proof}\normalfont
First we prove that $\text{VCdim}(f) \geq |\mathcal{P}^E_\text{max}|$, which means there exists a set of $|\mathcal{P}^E_\text{max}|$ examples $X_1,...X_{|\PEm|}$ that any labels $Y_1,...,Y_{|\PEm|}$ can be realized by a classifier $f$ built from $\mathcal{P}$. To construct this example set, we use the maximum efficient set $\PEm$. For any pattern $z_i$ in $\PEm$, since $\mathcal{I}^{\mathcal{S}}(z_i) \not\subset \mathcal{I}^{\mathcal{S}}(\PEm\backslash z_i)$, there always exists a data point $X_i \in \mathcal{S}$ that satisfies only $z_i$, i.e.,
$$X_i \in \mathcal{I}^{\mathcal{S}}(\PEm \backslash z) -  \mathcal{I}^{\mathcal{S}}(z),$$
for $i \in \{1,...,|\PEm|\}$.
Each $X_i$ is covered by exactly one pattern in $\PEm$. These points can always be shattered since for any labels, we can from a pattern set $A=\{z_i|z_i\in \PEm, \text{s.t. }\mathbbm{1}_{z_i}(X_i) = 1, Y_i = 1\}$. Therefore, all possible labels of $Y_1,...,Y_N$ can be realized, which means that $\text{VCdim}(f) \geq |\mathcal{P}^E_\text{max}|$.

Then we show $\text{VCdim}(f)\leq|\PEm|$. We prove this by contradiction. Let $\PEm$ be the maximum efficient set of $\mathcal{P}$.  Assume there exists a set of $h$ examples $X_1,...X_h$ where $h>|\PEm|$, and their labels $Y_1,...Y_h$ can always be realized. Let $\mathbf{0}_{\backslash i}$ denote an all-zero vector of size $h$ except a one at the $i$-th position. For $\mathbf{0}_{\backslash i}$ to be a realizable set of labels, there must exist a pattern $z_i$ that satisfies $\mathbbm{1}_{z_i}(X_i) = 1$ and  $\mathbbm{1}_{z_i}(X_j) = 0$ for $j \neq i$. This should be true for all $i \in \{1,...,h\}$. Therefore, there must exist $h$ such patterns that each of them covers a data point that only satisfies this pattern. According to definition \ref{def:es}, this is equivalent to declaring that these $h$ patterns is an efficient set, and the size of the set is $h$, which is greater than $|\PEm|$. This contradicts the assumption that $\PEm$ is the maximum efficient set and should contain the largest number of patterns. Therefore, $\text{VCdim}(f)\leq|\PEm|$.

Thus, we conclude that the VC-dimension of a classifier $f$ built from $\mathcal{P}$ is $|\PEm|$. (Learning an efficient set will be another topic that we do not discuss in this paper.)
\end{proof}
\subsection{Comparing with Other Discrete Classifiers}
Like OA classifiers, decision trees and random forests also discretize the input space and assign each subspace with a label. We prove that for these models, there always exist equivalent OA classifiers. These theorems are simple, but may not be obvious to those who have not thought about it.
We present the definition of two classifiers being \textit{equivalent} below.
\begin{definition}
Two classifiers $f_1, f_2$ are equivalent if for any input $X$, $f_1(X) = f_2(X)$.
\end{definition}
In a decision tree, the leaves divide up the input space into areas with different labels, which will be the predicted outcome for any data that ends up in that area. A path from the root to a leaf is a conjunction of literals, i.e., a pattern. See Figure \ref{fig:tree} as an example. 
\begin{figure}[h]
\centering
\includegraphics[width=0.45\textwidth]{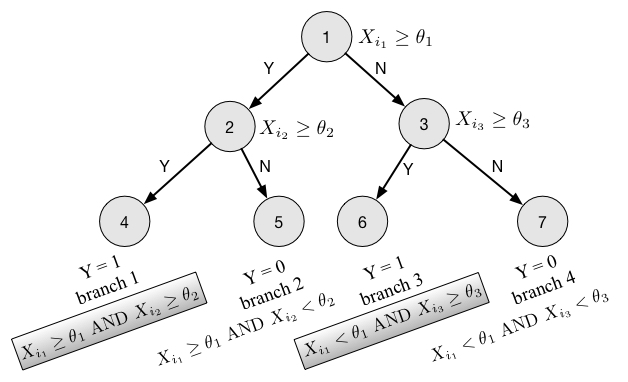}
\caption{A decision tree and the corresponding patterns.}
\label{fig:tree}
\end{figure}
The decision tree ends up with 4 leaves, and therefore 4 patterns. 
To convert the tree into an equivalent OA classifier, we simply collect the patterns that are associated with positive leaves, in this case, leaf 4 and 6, shown in grey boxes.

\begin{thm}\label{thm:equivalent_list}
For any decision tree $f_1$, there exists an equivalent OA classifier $f_2$, where the number of patterns in $f_2$ equals to the number of positive $Y$ labels in $f_1$.
\end{thm}
Similar inductions hold for random forests. Random forest is an ensemble method based on decision trees. If an input data point falls into a positive leaf in at least half of all the trees, then it is labeled as positive. Therefore, the equivalent OA classifier consists of patterns that are conjunctions of positive patterns from at least half of all the trees. We summarize the above statements into Theorem \ref{thm:equivalent}.
\begin{thm}\label{thm:equivalent}
For any random forest $f_1$, there exists an equivalent OA classifier $f_2$. If $f_1$ consists of $K_\text{rf}$ decision trees, and the $k$-th tree has $n_k$ positive leaves for $k \in \{1,...,K_\text{rf}\}$, then the size of the pattern set in $f_2$ is upper bounded by $\sum_{\pi \in \Pi}\prod_{k \in \pi} n_k$, where $\Pi$ is a collection of all possible combinations of $\lfloor \frac{K_\text{rf}}{2} \rfloor +1$ elements selected from $\{1,...,K\}$.
\end{thm}
The size is upper bounded by instead of exactly equal to  $\sum_{\pi \in \Pi}\prod_{k \in \pi} n_k$  because some patterns could be equivalent, or contained in others. Note that we only need conjunction of  $\lfloor \frac{K_\text{rf}}{2} \rfloor +1$ patterns because conjunctions of more than that are contained in conjunctions of exactly $\lfloor \frac{K_\text{rf}}{2} \rfloor +1$ positive patterns. 

The above theorems provide theoretical guarantees that OA classifiers can be as good as decision trees and random forests, in terms of predictive performance, although it might not be desired to create complex OA models, since the whole purpose of designing an OA model is to favor its interpretability over other models. 
\vspace{-2mm}
\section{Experiments} \label{sec:experiment}
Our experiments include applying OA models to diagnose obstructive sleep apnea (OSA) and experimenting on 9 public datasets from UCI Machine Learning repository \citep{Lichman:2013}.

To construct simple OA models for interpretability purposes, we set the maximum number of patterns to be 5 in all experiments. (In an OOA framework, we set $K=5$; in a OOAx framework, we add a constraint that the sum of $\zeta_k$'s is less than or equal to 5.) Since we placed strong restrictions on the characteristics of OA models, we expect to lose predictive accuracy over unrestricted baseline methods. In many of the experiments we did, we found that OA models do not lose in performance, and most of the time are the top performing models, while achieving a substantial gain in interpretability.

\vspace{-4mm}
\subsection{Diagnosing Obstructive Sleep Apnea}
The main experimental result is an application of OA models to build a  diagnostic screening tool based on routinely available medical information. We analyzed polysomnography and self-reported clinical information from 1922 patients tested in the clinical sleep laboratory of Massachusetts General Hospital, first analyzed in 2015 \citep{UstunRu15,UstunEtAl15}. The goal is to classify which patients who enter into the MGH Sleep Lab have OSA, based on a survey filled out upon admission.  We produce predictive models for OSA screening using attributes from self-reported symptoms and self-reported medical information. The attributes include detailed information such as age, sex, BMI, sleepiness, if the patient snores, if the patient wakes up during sleep, if the patient falls back to sleep easily, the level of tiredness, etc. The data set was binary coded into 112 attributes.

Due to the size of this dataset, we chose OOAx for faster computation. We mined patterns with minimum support of 5\% and maximum length of 3. We tuned parameters $C_1$ and $C_2$ using nested cross-validation to obtain the best performance, under the constraint that the pattern size cannot exceed 5. We measured out-of-sample performance using accuracy from 5-fold cross validation for OA models and 5 other methods that adhere to a certain level of interpretability, BOA \citep{wang2015or}, Lasso, C4.5, CART and RIPPER. For all baseline methods, we tuned the hyperparameters  with grid search in nested cross validation. The results are displayed in Table~\ref{tab:sleep}.

\begin{table}[h]
\setlength\tabcolsep{4pt}
\centering
\caption{Accuracy comparison for OA models and baselines on obstructive sleep apnea dataset.}
\label{tab:sleep}
\begin{tabular}{l|c|l}
\toprule 
\multicolumn{1}{c|}{} & \multicolumn{1}{c|}{Accuracy}          & \multicolumn{1}{c}{ Complexity }                                                                                                                                     \\ \hline
OOAx                    & .80(.01)                      & \begin{tabular}[c]{@{}l@{}}total number of patterns = 2.8\\ average length of patterns = 1.67 \\ total number of literals = \textbf{4.7}\end{tabular}                      \\ \hline
BOA           & .80(.01)               & \begin{tabular}[c]{@{}l@{}}total number of patterns = 4\\ average length of patterns = 1.75 \\ total number of literals =\textbf{7}\end{tabular}                        \\ \hline
RIPPER                & \multicolumn{1}{c|}{.79(.01)} & \multicolumn{1}{l}{\begin{tabular}[l]{@{}l@{}}total number of rules = 4.6\\ average length of rules = 4 \\ total number of literals = \textbf{18.4}\end{tabular}} \\ \hline
C4.5                  & .79(.01)                      & \begin{tabular}[c]{@{}l@{}}depth = 5\\ total number of nodes = \textbf{14.4}\end{tabular}                                                            \\ \hline
CART                  & .79(.01)                      & \begin{tabular}[c]{@{}l@{}}depth = 5\\ total number of nodes = \textbf{11}\end{tabular}                                                              \\ \hline
Lasso                 & .80(.01)                      & non-negative coefficients = \textbf{5}                                                                                                                                      \\ \bottomrule
\end{tabular}
\end{table}
To compare interpretability, we reported the complexity of each model averaged across 5 folds. For OA and BOA models, we reported the total number of patterns, average length of patterns and the total number of literals. OA models achieve the same performance as BOA models but with higher interpretability. This is due to a more flexible control over the size and shape of the pattern set compared to BOA models. RIPPER models are decision lists, having a different form than Or's of And's. 
We reported the total number of rules, average length of rules and total number of literals. For decision trees C4.5 and CART, we reported the depth of a tree, and the total number of nodes in a tree. For lasso, we reported the number of non-negative coefficients. 
Since baseline models have different logical forms than OA models, we compare one universal metric, the number of literals/nodes used in each model, marked in bold in Table~\ref{tab:sleep}. We find that OA models used substantially fewer literals than all other models while achieving a competitive accuracy to all models.
\setlength\tabcolsep{3pt}
\begin{table*}[t]
\centering
\caption{Accuracy comparison for OA models and baselines on UCI datasets.}
\label{uci}
\begin{tabular}{l|c|ccccccc|cc}
\toprule
\multicolumn{1}{l|}{\multirow{2}{*}{}} & \multirow{2}{*}{\begin{tabular}[c]{@{}c@{}}Data \\ Type\end{tabular}} & \multicolumn{7}{c|}{{\it Interpretable Models}}& \multicolumn{2}{l}{{\it Uninterpretable Models}} \\
\multicolumn{1}{l|}{}                  &                                                                       & OOA          & OOAx            & BOA          & Lasso    & C4.5     & CART     & RIPPER   & random forest           & SVM                     \\ \hline
blogger                                & \multirow{4}{*}{Categorical}                                          & .85(.11)     & {\bf .86(.10)} & .80(.04)     & .81(.08) & .77(.05) & .78(.07) & .76(.06) & .82(.07)                & .82(.10)                \\
votes                                  &                                                                       & {\bf .98(.02)}     & {\bf .98(.02)} & .95(.02)     & .96(.02) & .96(.02) & .96(.03) & .96(.01) & .95(.02)                & .97(.02)                \\
tic-tac-toe                            &                                                                       & {\bf 1(.00)} & {\bf 1(.00)}   & {\bf 1(.00)} & .71(.02) & .92(.03) & .93(.02) & .98(.01) & .99(.00)                & .99(.00)                \\
monks1                                 &                                                                       & {\bf 1(.00)} & {\bf 1(.00)}   & {\bf 1(.00)} & .76(.02) & .90(.06) & .88(.07) & .94(.12) & {\bf 1(.00)}            & {\bf 1(.00)}            \\ \hline
bupa                                   & \multirow{4}{*}{Numerical}                                            &   .65(.02)           & .65(.03)       & .66(.02)     & .68(.04) & .63(.04) & .68(.03) & .65(.05) & .70(.03)                & {\bf .73(.04)}          \\
transfusion                            &                                                                       &  .78(.02)            & {\bf .80(.01)} & .77(.01)     & .77(.02) & .76(.02) & .78(.02) & .78(.02) & .78(.02)                & .80(.02)                \\
banknote                               &                                                                       &  .98(.01)          & .97(.01)       & .96(.01)     & 1(.00)   & .90(.01) & .90(.02) & .91(.01) & .91(.01)                & {\bf 1(.00)}            \\
indian-diabetes                        &                                                                       &  .73(.02)        & {\bf .77(.03)} & .74(.02)     & .67(.01) & .66(.03) & .67(.01) & .67(.02) & .76(.02)                & .69(.01)                \\ \hline
heart                                  & mixed                                                                 &   .80(.04)       & .84(.05)       & .83(.07)     & .85(.04) & .76(.06) & .77(.06) & .78(.04) & .81(.06)                & {\bf .86(.06)}                            
\\
\bottomrule
\end{tabular}
\end{table*}

An example of an OA model is shown below.
\begin{algorithmic}
 \If { a patient satisfies (age $\geq 30$ {\em AND} patient checked snoring as a potential symptom in the questionnaire), \\ \quad {\em OR}  (age $\geq 30$ {\em AND} patient checked snoring as a reason for "why are you here" in the questionnaire),\\ \quad {\em OR}  (age $\geq 30$ {\em AND} has hypertension), \\ \quad {\em OR} (BMI $\geq 25$)}
 \State predict the patient has sleep apnea,
 \Else
 \State predict the patient does not have sleep apnea.
 \EndIf
\end{algorithmic}
The model lists four patterns to characterize patients that has sleep apnea. It is a sparse model with only a few attributes and a simple structure, and can potentially be used by people without a machine learning background.
\vspace{-4mm}
\subsection{Performance on UCI Datasets}
We applied OOA and OOAx to several UCI datasets and compared with 5 previously mentioned interpretable models and 2 black box models, random forest and SVM. In the experimental set up, we set a time limit for the MIP in the OOA framework to ensure that it returns a solution in a reasonable amount of time.
Table \ref{uci} displays the mean and standard deviation of out-of-sample accuracy across 5 folds. 

We observed that even with the severe restrictions, OA classifiers achieve very competitive performance. For the four categorical datasets in Table \ref{uci}, OA classifiers always do better than other models. Especially for tic-tac-toe and monks, where there are correct models that correctly classify all examples, OA models are able to discover the correct patterns and achieve 100\% accuracy. For numerical and mixed datasets, OA models' performance levels are on par with those of other methods, sometimes slightly dominated by uninterpretable machine learning models. 


We show an example of an OA classifier learned from  dataset ``votes'' using OOA framework. This data set includes votes for each of the U.S. House of Representatives Congressmen on 16 key votes on water project cost sharing, duty free exports, immigration, education spending, anti-satellite test ban and etc. The objective is to predict if the voter is democratic or republican. 
\begin{algorithmic}
 \If { a voter (votes for eduction spending {\em AND} for physician fee freeze {\em AND} against water project cost sharing), \\ 
 \quad {\em OR}  (votes for  export administration act of South Africa {\em AND} for physician fee freeze {\em AND} agains synfuels corporation cutback),\\ 
 \quad {\em OR}  (votes against  aid to Nicaraguan Contrast {\em AND} against  adoption of the budget resolution {\em AND} against handicapped infants and toddlers act {\em AND} against superfund right to sue), \\ 
 \quad {\em OR} (votes for  adoption of the budget resolution {\em AND} for  physician fee freeze {\em AND} agains synfuels corporation cutback),\\ 
  \quad {\em OR} (votes against  adoption of the budget resolution {\em AND} for  El Salvador aid {\em AND} for physician fee freeze),\\
  \quad {\em OR} (votes for aid to Nicaraguan Contras {\em AND} against adoption of the budget resolution {\em AND} against duty free exports {\em AND} against synfuels corporation cutback),}
 \State predict the voter is republican,
 \Else
 \State predict the voter is democratic.
 \EndIf
\end{algorithmic}

\vspace{-4mm}
\section{Conclusion}
OA models have a long history. They are particularly useful as either (i) interpretable screening mechanisms, where they reduce much of the data from consideration from a further round of modeling, and (ii) consideration sets from marketing, which are rules that humans create to reduce cognitive load in order to make a decision.

We presented two optimization-based frameworks for learning Or's of And's. 
The first framework, OOA, uses a MIP to directly form patterns from data. It can deal with both categorical and numerical data without pre-processing. The second framework OOAx reduces computation through pre-mining patterns. We provided bounds on the support of patterns that guarantee that the pattern space can be safely reduced. Both methods can produce high quality OA classifiers, as demonstrated through experiments. They achieve competitive performance compared to other classifiers, with a substantial gain in sparsity and interpretability. 

One of the main benefits not discussed extensively earlier is the benefit of customizability. Because we use MIP/ILP technology, constraints of almost any kind are very easy to include, and we do not need to derive a new algorithm; this benefit does not come with any other technology that we know of. Customizability is an important component of interpretability.


\newpage
\bibliography{mboa,rules} 
\bibliographystyle{aaai}
\newpage
\section*{SUPPLEMENTARY MATERIAL}
\begin{proof}(Of Theorem \ref{thm:neg_supp})
The objective function of the optimal solution $A^*$ is
\begin{align*}
&L(A^*) =  \frac{N^+ -\text{supp}^{S^+}(A^*) + \text{supp}^{S^-}(A^*) }{N} + \\
& C_1\#\text{literals}(A^*) +  C_2 \#\text{patterns}(A^*) \notag \\
\geq &\frac{N^+ -\text{supp}^{S^+}(A^*) + \text{supp}^{S^-}(A^*) }{N} + C_1 + C_2 \notag \\
\geq &\frac{\text{supp}^{S^-}(A^*) }{N} + C_1 + C_2.
\end{align*}
These inequalities become tight when $A^*$ was one pattern with one literal covers the whole positive class and some of the negative class.
Let $\emptyset$ denote an empty set where there are no patterns and all the data points are classified as negative, so the total number of errors are the number of positive data, denoted as $N^+$. The objective function given $\emptyset$ is
\begin{equation*}
L(\emptyset) = \frac{N^+}{N}.
\end{equation*}
Since $A^* \in \arg\min_{A} L(A)$, $L(A^*) \leq L(\emptyset)$, then
\begin{equation*}
\text{supp}^{S^-}(A) \leq N^+ -N\left(C_1+C_2 \right)
\end{equation*}
Since $I^{S^-}(A)  = \cup_{a \in A}I^{S^-}_a$, then $\text{supp}^{S^-}(z)\leq \text{supp}^{S^-}(A^*) $, thus
\begin{equation*}
\text{supp}^{S^-}(z) \leq N^+ -N\left(C_1+C_2 \right)
\end{equation*}
\end{proof}
\begin{proof}(Of Theorem \ref{thm:pos_supp})
The worst case when pattern $z$ is removed is when $z$ is an accurate rule with confidence equal to 100\%, i.e., all data points that satisfy pattern $z$ are positive; and the points covered by $z$ are not covered by any other pattern. Therefore once removing it, the number of errors increased by the positive support of z. On the other hand, removing $z$ benefits the regularization terms, by decreasing the sum of pattern lengths by at least 1, and the number of patterns by 1. Then the objective function given $A_{\backslash z}$ obeys
\begin{align*}
L(A_{\backslash z}) = & L(A) + \frac{\#\text{error}(A_{\backslash z}) - \#\text{error}(A)}{N} + \\
& C_1\left(\#\text{literals}(A_{\backslash z})-\#\text{literals}(A)\right) + \\
& C_2\left( \#\text{patterns}(A_{\backslash z})-\#\text{patterns}(A)\right)\\
\leq & L(A) +\frac{\text{supp}^{S^+}(z)}{N} - C_1 -C_2.
\end{align*}
In order to prove $L(A_{\backslash z}) \leq L(A)$, we need $$L(A) + \frac{\text{supp}^{S^+}(z)}{N} - C_1 -C_2 \leq L(A),$$ i.e., $$\text{supp}^{S^+}(z)\leq \left(C_1+C_2\right)N.$$
\end{proof}
\begin{proof}(Of Theorem \ref{thm:size})
Let $M^* = |A^*|$. $A^*$ contains $M^*$ patterns where each pattern has at least one literal. Therefore, the objective function given $A^*$ is lower bounded by
\begin{align*}
L(A^*) & \geq \frac{\#\text{error}(A^*)}{N} + C_1M^* + C_2 M^*  \\
& \geq M^* \left(C_1 +C_2\right) .
\end{align*}
Since $A^* \in \arg\min_{A} L(A)$, $L(A^*)\leq L(\emptyset)$. That is
\begin{equation*}
M^* \left(C_1 +C_2\right)\leq \frac{N^+}{N}.
\end{equation*}
Therefore
\begin{equation*}
M^* \leq \frac{N^+/N}{C_1+C_2}.
\end{equation*}
\end{proof}
\end{document}